\documentclass{article}
\usepackage{amsmath,amssymb,amscd,amsfonts,amsthm}
\usepackage{comment,mathtools}
\usepackage{url}
\usepackage{xcolor}
\usepackage{comment}
\usepackage{paralist} 
\usepackage{bm}
\usepackage{graphicx}
\usepackage{booktabs}
\usepackage[round]{natbib} 
\usepackage{amsfonts}
\usepackage{geometry}
\usepackage{authblk}

\renewcommand{\emptyset}{\varnothing}
\renewcommand{\le}{\leqslant}
\renewcommand{\ge}{\geqslant}

\newcommand{\bsx}{\boldsymbol{x}}

\newcommand{\bsz}{\boldsymbol{z}}
\newcommand{\real}{\mathbb{R}}
\newcommand{\cx}{\mathcal{X}}
\newcommand{\bszero}{\boldsymbol{0}}
\newcommand{\bsone}{\boldsymbol{1}}

\newcommand{\e}{\mathbb{E}}

\newcommand{\rd}{\,\mathrm{d}}

\newcommand{\phm}{\phantom{-}}
\newcommand{\phz}{\phantom{0}}

\newcommand{\giv}{\!\mid\!}
\newcommand{\dunif}{\mathsf{U}}

\newcommand{\tran}{\mathsf{T}}

\usepackage{pifont}
\newcommand{\cmark}{\ding{51}}%

\newtheorem{theorem}{Theorem}

\theoremstyle{definition}
\newtheorem{remark}{Remark}  
\newtheorem{corollary}{Corollary}

\author[1]{Naofumi Hama}
\author[1]{Masayoshi Mase}
\author[2]{Art B. Owen}
\affil[1]{Hitachi, Ltd.\\Research \& Development Group}
\affil[2]{Department of Statistics\\Stanford University}
\affil[ ]{\textit {naofumi.hama.hd@hitachi.com, masayoshi.mase.mh@hitachi.com, owen@stanford.edu}}

\date{April 2023}
\title{Model free variable importance for high dimensional data}
\begin{document}
\maketitle
\begin{abstract}

    A model-agnostic variable importance method
    can be used with arbitrary prediction functions.  Here we
    present some model-free
    methods that do not require access to the prediction
    function. This is useful when that function is
    proprietary and not available, or just extremely
    expensive.  It is also useful when studying
    residuals from a model.
    The cohort Shapley (CS) method is
    model-free but has exponential cost in the dimension
    of the input space.
    A supervised on-manifold Shapley method from \cite{frye2020shapley}
    is also model
    free but requires as input a second black box model
    that has to be trained for the Shapley value problem.
    We introduce an
    integrated gradient (IG) version of cohort Shapley,
    called IGCS, with cost $\mathcal{O}(nd)$.
    We show
    that over the vast majority of the relevant
    unit cube that the IGCS value function is close
    to a multilinear function for which IGCS matches
    CS.  Another benefit of IGCS is that is allows
    IG methods to be used with binary predictors.
    We use some area between curves (ABC) measures to
    quantify the performance of IGCS.  On a
    problem from high energy physics we verify
    that IGCS has nearly the same ABCs as
    CS does.  We also use it on a problem from
    computational chemistry in 1024 variables.
    We see there that IGCS
    attains much higher ABCs than we get from
    Monte Carlo sampling.
    The code is publicly available at \url{https://github.com/cohortshapley/cohortintgrad}.

\end{abstract}

\section{Introduction}

Quantifying the importance of predictor variables is a first
step in explainable AI.  As noted by \cite{hooker2018benchmark},
and others, there is no ground truth notion of importance. One
must instead choose a definition based on such factors as
computational feasibility, whether it is defined in terms of `off manifold' data
when the underlying features exhibit strong dependence,
whether a model must be differentiable,
and whether the method can attribute importance
to an unused variable which correlates with one that is used,
and yet more factors.
People can reasonably disagree
about whether some of these properties are beneficial
or constitute flaws, and we include some discussion of those
points below.

Methods from game theory, notably the Shapley value \citep{shap:1953} and the Aumann-Shapley value \citep{auma:shap:1974} are becoming widely used in defining measures
of variable importance.
When one can reasonably cast the variable importance problem in terms
of a game, then
there are persuasive (though still debatable) axioms that
lead to a unique definition.
The SHapley Additive exPlanations (SHAP)
approach of \citep{lundberg2017unified} and
the Integrated Gradients (IG)
of \cite{sundararajan2017axiomatic}
are widely cited methods derived from Shapley and Aumann-Shapley
values, respectively.

We assume throughout that there are predictor variables
$\bsx_i$ with $d$ components (not necessarily real-valued),
a response $y_i\in\real$ and a function $f$ where $f(\bsx_i)\in\real$
is the prediction of $y_i$.  There are subjects $i=1,\dots,n$
and for one or more target subjects $t$, we want to quantify
how important $x_{tj}$ is for $f(\bsx_t)$.

One important distinction between methods is whether they are
model-based,  model-agnostic, or model-free.
For example, TreeSHAP \citep{lundberg2020local} can
be computed very efficiently but only for models that have
a tree structure, so it is model-based.
By contrast, techniques such
as the baseline Shapley method in \cite{sundararajan2020many}
can be used on any black
box model $f(\cdot)$, and hence are called model-agnostic.
In this paper we consider model-free variable importance
methods that use only the observed values $(\bsx_i,f(\bsx_i))$
or even $(\bsx_i,y_i)$
for $i=1,\dots,n$ to quantify the importance of components
$x_{tj}$ on the prediction $f(\bsx_t)$ or the response $y_t$.

The cohort Shapley (CS) method of \cite{mase2019explaining} is
model-free. However, exact computation of the CS values
has a cost that is exponential in the number $d$ of variables
within $\bsx_i$.
Integrated gradient methods have been proposed for Shapley
value by \cite{sundararajan2017axiomatic}.
These are versions of the Aumann-Shapley value \citep{auma:shap:1974}.
In this paper we construct an integrated gradients
approximation to CS, which we call IGCS.
IGCS has a cost that is $O(nd)$ for $n$ observations.
We found one other model-free method,
namely the surrogate Shapley method of \cite{frye2020shapley}.
More precisely, it
is a framework for constructing
model-free methods, rather than a model-free method itself, as we explain in Section~\ref{sec:rel}, and it has exponential cost.
While this paper was being reviewed,
another model-free approach was proposed.
\cite{ahn2023local} use local polynomial approximation
and bootstrapping to judge variable importance
for a prediction.

There are several reasons to select a model-free method.
\cite{mase2021fairness} study the COMPAS model
predicting recidivism, from
\cite{angw:2016}.
In this case, only the predictions themselves are available
because the underlying prediction function is proprietary and not
available to researchers.  As a result, methods that require
access to the function $f(\cdot)$ cannot be used but CS is
model-free and hence could be used to address some algorithmic
fairness issues.
A second reason to use model-free methods is to study variable
importance for given data $(\bsx_i,y_i)$ without reference
to any model.  A variable that is important for predicting $y_i$
is then one that we reasonably expect to play a role in a good
prediction model for $y$.  A third motivation for model-free
methods is to learn what variables are important for
residuals $y_i -f(\bsx_i)$ or the absolute values
or squares of those residuals.
Even when we have access to $f(\cdot)$
we still do not have a way to study residuals
at points $\bsx$ where we do not have the corresponding $y$
value.  If a variable is important for residuals, that provides
a hint that the model has some flaw that can be improved.
In the given data framework it is possible that two or
more observations have identical $\bsx$ but different $y$.
Both CS and IGCS can handle this setting.  That is, in the
list of $(\bsx_i,y_i)$ pairs it is not necessary for $y$
to be a function of $\bsx$.

While the original IG method requires differentiability
of the prediction function $f$, IGCS does not.
It is based on weighted means of observed function
values and the necessary gradient is defined in
terms of a differentiable weight function.
\cite{hama2022deletion} describe an interpolation
method for IG with binary variables, but it
has exponential cost unlike IGCS.
Indeed, as noted above, IGCS only needs
values $f(\bsx_i)$ or simply $y_i$ and $\bsx_i$
can be anything that $f$ can produce a real value for.

The remainder of this paper is organized as following.
Section~\ref{sec:rel} cites some related
works in the explainable AI (XAI)
literature.
Section~\ref{sec:formalization} introduces
some notation,
reviews Shapley value, cohort Shapley and integrated gradients.
It defines our IGCS proposal by using soft versions
of the similarity function in cohort Shapley.
We also introduce what to our knowledge
is a more general setting where integrated gradients
match Shapley value beyond the settings from
\cite{owen1972multilinear} or \cite{sundararajan2017axiomatic}.
Section~\ref{sec:asymptote} shows via Taylor
expansion that the value function underlying
IGCS is very nearly one where IGCS matches
CS over the vast majority of its input
space provided that the dimension $d$
satisfies $d\gg\log(n)$ for $n$ data points.
This finding is obtained under very general conditions
on the data, and does not need assumptions strongly
tuned to the model form that the usual integrated gradients
method does.
Section~\ref{sec:experiment} conducts experiments to confirm the feasibility of IGCS. In a 16-dimensional
problem from high energy physics it performs almost
as well as CS.  In a 1024-dimensional problem from
computational chemistry it performs much better than
a Monte Carlo strategy and the results we inspected
made sense in terms of basic chemistry.
There are some conclusions in Section~\ref{sec:conclusion}.
Some appendices provide further details and results.

\section{Related Works}
\label{sec:rel}
We are proposing a variable importance
measure that is model-free and scales to large dimensions.
Here we outline some parts of the enormous literature
on variable importance methods, emphasizing
methods related to our proposal.
We introduce some concepts that distinguish variable
importance methods from each other and we describe
some of the key properties of those methods and
tabulate which methods in a comparison group have
those properties.

To get a sense of the size of that literature,
the survey of \cite{wei:lu:song:2015} on variable
importance in statistics has 197 references of which
24 are themselves surveys and  \cite{raza:etal:2021} cite
over 350 references in global sensitivity analysis.  For
surveys in machine learning, see \cite{moln:2018}
and \cite{chen2022explaining} who focus on methods that like
our proposal have a rationale from game theory.
There is even a debate over whether models should be
explained at all, instead of just using models that are
themselves interpretable \citep{rudin2019stop}.
However, given that hard to interpret models are widely
used, we think it makes
sense to seek interpretations of them.

At first sight, variable importance should be simple.
We can change the inputs to $f$ as we like and
recompute it.
A variable that when changed brings large changes to
$f$ is more important than one that brings small changes
or no changes at all.  The difficulty, and the reason
for the vast literature, is that there is a combinatorially
large number of possible ways to do this. We could make
small or large changes to each $x_{tj}$. For each $j'\ne j,$ we
could leave $x_{tj'}$ fixed while changing $x_{tj}$
or we could change it
in one of several ways.  We need to choose a way
to aggregate importance over all the changes we will make
to $x_{tj}$ and we may need to choose how
to aggregate again over all $n$ target
inputs $\bsx_t$ in our problem.  These approaches typically
differ in how they define importance, not how they
estimate some well-accepted definition of importance.
This is why there is no ground truth for variable importance
\citep{hooker2018benchmark}.  Choosing a
correct definition is, as \cite{noimpossible} remark, a problem
of identifying the causes of some effects, instead of the more usual
machine learning setting where we study the effects of causes.
\cite{dawi:musi:2021} describe the extreme difficulty that
identifying causes of effects brings to data science.
If $f(\bsx_t)=1$ instead of $0$, there can be many different
changes to $\bsx_t$, affecting one or more components
$x_{tj}$ that would have brought the other outcome.  Accounting
for all of those counterfactuals is what makes it hard to
explain why $f(\bsx_t)=1$ instead of $0$.

The choice of a variable importance method can be based on how
well some properties of the underlying or implied definition
align with our goals.  For large $d$, we would prefer a method
with sub-exponential cost.
We might prefer a method that obeys some intuitive
axioms, like those underlying the Shapley value (see Section~\ref{sec:shapley}).
When studying algorithmic fairness we would want
a method that can detect redlining by attributing
importance to variables not actually used in $f$,
whereas some global sensitivity analysis
researchers do not want that \citep{heri:etal:2022}.
See also \cite{kuma:2020}.

A variable importance measure can be `local',
explaining the prediction $f(\bsx_t)$ for one observation
$t$ or `global', explaining the
value of variable $j$ in getting an accurate model
for all observations $i=1,\dots,n$.  We focus on local methods.
A method can be based on fixing some components of $\bsx_t$
while intervening to vary the others. Such methods
are called `interventional'. Or, a method can be based
on fixing some set of components and summarizing
the effects of randomly sampling the others. We might for example,
report the conditional mean of $f(\bsx_t)$ given the values
of the fixed components.  Such methods are called `conditional'.
We prefer conditional methods but we include some
interventional ones in our comparisons.
A conditional method can be
`on-manifold' attempting to only use $f$ at realistic inputs
$\bsx$, such as values scattered around a manifold within
the input space, or it can be `off-manifold', computing with arbitrary combinations of input components.

A problem with off-manifold methods is that
they can use problematic input combinations
unlike anything seen in past data or
even possible in future data.
Both~\cite{hooker2019please} and \cite{mase2019explaining}
emphasize this point.
On the other hand, on-manifold methods are more challenging to compute.
Because it can be hard to be sure where exactly
the data manifold lies in high dimensional settings,
methods that use only actually observed data
can more reliably avoid problematic inputs.

Another criterion is whether the method to explain the
importance of a variable in a black box function
makes use of some other black box function in order
to choose some counterfactual inputs.  We prefer
to avoid second black boxes as they may require
their own explanations.

Table~\ref{tab:affordances} shows a grid of properties
of some variable importance measures
with game theoretic motivations.
There is not space to describe all of the methods.
For a recent survey of the area see \cite{chen2022explaining}.
We say a bit more about some methods in Section~\ref{sec:formalization},
when we introduce our proposed method.

\begin{table}[t]
    \centering
    \begin{tabular}{lcccccccccc}
        \toprule
        Property         & CS     & BS     & TS     & KS     & IG     & OM     & SS     & CKS    & eGKW            
                         & IGCS                                                                                    \\
        \midrule
        Real obs. only   & \cmark &        & \cmark &        &        & \cmark & \cmark &        & \cmark & \cmark \\
        $x$ dependence   & \cmark &        &        &        &        & \cmark & \cmark & \cmark & \cmark & \cmark \\
        Unused variables & \cmark &        &        &        &        & \cmark & \cmark &        & \cmark & \cmark \\
        Model-agnostic   & \cmark & \cmark &        & \cmark &        & \cmark & \cmark & \cmark & \cmark & \cmark \\
        Model-free       & \cmark &        &        &        &        &        & \cmark &        & \cmark & \cmark \\
        Sub-exponential  &        &        & \cmark & \cmark & \cmark &        &        & \cmark & \cmark & \cmark \\
        Exact Shapley    & \cmark & \cmark & \cmark &        &        &        &        &        &                 \\
        No 2nd black box & \cmark & \cmark & \cmark &        & \cmark &        &        &        &        & \cmark \\
        Automatic        &        &        & \cmark &        &        &        &        &        &        &        \\
        \bottomrule
    \end{tabular}
    \caption{
        \label{tab:affordances}
        Summary of properties of variable importance methods: cohort Shapley \citep{mase2019explaining}, baseline Shapley \citep{sundararajan2020many, chen2022explaining}, TreeSHAP \citep{lundberg2020local},  KernelSHAP \citep{lundberg2017unified},  integrated gradients \citep{sundararajan2017axiomatic}, On-manifold \citep{frye2020shapley}, surrogate SHAP \citep{frye2020shapley} and Conditional Kernel SHAP \citep{aas2021explaining}, empirical Gaussian Kernel Weight (Appendix~\ref{app:egkw})
        and our integrated gradient cohort Shapley proposal.
        Properties from top to bottom:
        uses only observed predictor combinations,
        uses no independence assumption on inputs,
        can attribute importance to unused variables,
        can use any $f$, does not need $f$,
        cost of exact computation is not
        exponential in $d$,
        computes a Shapley value instead of
        an approximation,
        does not use a second black box to explain $f$, and
        requires no user input.
    }
\end{table}

The supervised on-manifold method of~\cite{frye2020shapley},
labeled SS in Table~\ref{tab:affordances} is a framework
for constructing model-free Shapley values.
They train a neural network to predict $f(\bsx_i)$
given predictors $\tilde\bsx_i(u)$ where $u$ is a
set of predictor indices.
Here $\tilde\bsx_i(u)$ is a vector formed by
replacing $\bsx_{ij}$ by an NA value
(not available) for
all predictors $j\in u$.  In the notation
we introduce in Section~\ref{sec:formalization},
$\tilde\bsx_i(u)=\bsx_{0,u}{:}\bsx_{i,-u}$
where $\bsx_0$ is a tuple of $d$ NA values.
Their NA value was $-1$
which never appears in the data sets they study.
They train with a loss function based on accurately
predicting $f(\bsx_i)$ given, uniformly distributed
data points $i$, a Shapley value-derived distribution on $u$
and a distribution for $x_{ij}$ when $j\not\in u$.
There are many choices for the neural network to use,
the choice could be customized to the data set at
hand, and their framework could be used with other
algorithms other than neural networks.  They train with
squared error loss, which is natural given that the
Shapley value they want is defined as an expectation.
We label this method as exponential as
there is an exponentially large number of subsets
$u$ to consider and corresponding models to train.

The empirical GKW method is our adaptation
of the Gaussian kernel weight method of \cite{aas2021explaining}.
See Appendix~\ref{app:egkw}.  The original GKW is
model-agnostic but not model-free.

The statistics literature has several measures of variable
importance that can be defined in terms of the joint distribution of
random variables without requiring a parametric form for any
predictions, and they are in that sense
model-free. These include the maximal correlation
of \cite{renyi1959measures}, the energy statistics of
\cite{szekely2013energy}, the conditional dependence measures
of \cite{azadkia2021simple}, the floodgate method of
\cite{zhang2020floodgate}.
These are commonly defined through conditional variances
and do not give additive local explanations of $f(\bsx_t)$ for
a given subject $t$, so they don't fit into our
comparisons.

\section{Notation, Background and the Method}\label{sec:formalization}

Variable importance measures
study a mapping $f:\cx\to\real$ that makes a
prediction based on input data $\bsx\in\cx = \prod_{j=1}^d\cx_j$.
A generic point $\bsx\in\cx$ is written $\bsx=(x_1,x_2,\dots,x_d)$.
We have $n$ observations $\bsx_1,\dots,\bsx_n$
and the goal is to quantify importance of the $d$ components
of $\bsx_t$ to $f(\bsx_t)$ (or $y_t$) where $t$ is the `target'
observation.
Because we are working in a model-free setting we
only need to compute with the values $(\bsx_i,f(\bsx_i))$
for $i=1,\dots,n$.  We can also relax the assumption
that $f$ is a mapping.  For instance, we can replace
values $f(\bsx_i)$ by values $y_i$ where $\bsx_i=\bsx_{i'}$
need not imply $y_i=y_{i'}$.

We use $[d]$ for the set $\{1,\dots,d\}$.
For $u\subseteq[d]$ we write $|u|$ for the cardinality
of $u$ and $-u$ for $[d]\setminus u$. For singletons
we often write $-j$ in subscripts instead of $-\{j\}$.
The subvector $\bsx_u$ has all components $x_j$
for $j\in u$ and no others.  We write $\bsx_{tu}$
for the corresponding subvector of $\bsx_t$.
Sometimes we put a comma
in the subscripts, as in $x_{t,-u}$ to make it clear
that we are not subtracting $u$ from $t$.
We need some notation for hybrid points
that merge components from two input vectors.
The point denoted by $\bsx_u{:}\tilde\bsx_{-u}\in\cx$
has components $x_j$ for $j\in u$ and $\tilde x_j$
when $j\not\in u$.  For instance, it is what you get
if you start with $\tilde\bsx$ and substitute $x_j$ for $\tilde x_j$
for every $j\in u$.

Because we are attributing importance of $f(\bsx_t)$
to the components $x_{tj}$ for $j=1,\dots,d$ the problem
is one of `local importance'.  As noted above,
this is distinct from
`global importance' problems of identifying which variables
must be included in the model in order to get accurate
predictions.

We will use a weight space $[0,1]^d$. In this space
we write $\bsone=(1,1,\dots,1)$ and $\bszero=(0,0,\dots,0)$.
Those two points belong to $\{0,1\}^d$ which we describe as
the set of `corners' of $[0,1]^d$. A general corner is
of the form $\bsone_u{:}\bszero_{-u}$ for $u\subseteq[d]$.

\subsection{Shapley Value and CS}\label{sec:shapley}
We begin by describing the Shapley value that underlies
the game theoretic methods.
In cooperative game theory, we suppose that a team of
$d$ players has created some value that we must
attribute to them as individuals. We assume
that we know the value $\nu(u)$ that would have been
created by any subteam $u\subseteq[d]$ of those participants
and suppose for now that $\nu(\emptyset)=0$.
If $j\not \in u$ then the incremental value from adding player
$j$ to team $u$ is
$$\nu(j\giv u)=\nu(u\cup\{j\})-\nu(u).$$
The Shapley value for player $j$ is the following weighted sum of its incremental values:
\begin{align}\label{eq:shapval}
    \phi_j = \frac1d\sum_{u\subseteq-j}
    {d-1\choose |u|}^{-1}\nu(j\giv u).
\end{align}
We can also get this value by building a team from
$\emptyset$ to $[d]$ in one of $d!$ orders, adding
one player at a time, and taking $\phi_j$ to be the
average incremental value from the addition of player $j$
over those $d!$ orders.

Shapley derived the value $\phi_j$ in~\eqref{eq:shapval}
as the unique solution compatible with four axiomatic criteria:
\begin{compactenum}
    \item {\bf Efficiency} $\sum_j\phi_j=\nu([d])$.
    \item {\bf Dummy} If $\nu(j\giv u)=0$ for all $u\subseteq-j$ then $\phi_j=0$.
    \item {\bf Symmetry} If $\nu(j\giv u)=\nu(j'\giv u)$ when
    $u\subseteq-\{j,j'\}$ then $\phi_j=\phi_{j'}$.
    \item {\bf Additivity} If $\nu$ and $\nu'$
    have values $\phi_j$  and $\phi'_j$ then $\nu+\nu'$
    has values $\phi_j+\phi'_j$.
\end{compactenum}
\smallskip
We find it convenient to drop the assumption that
$\nu(\emptyset)=0$. Then the efficiency condition
generalizes to $\sum_{j=1}^d\phi_j = \nu([d])-\nu(\emptyset)$.
That is, we are still explaining $\nu([d])-\nu(\emptyset)$
without being concerned over whether $\nu(\emptyset)=0$.
Shapley value provides an additive attribution because
the individual values $\phi_j$ sum to
the quantity being explained.

Using Shapley value we can derive importances for all $d$
variables by analogy where $\nu(u)$ is a measure of how
well the variables in $\bsx_u$ explain $f(\bsx)$.
For instance the LMG measure of \cite{lind:mere:gold:1980}
is the Shapley value when $\nu(u)$ is the $R^2$ value in
a linear model relating $y$ to $\bsx_u$. This is a global
measure of variable importance for all observations
not just the target $\bsx_t$.

An easily described local measure is baseline Shapley
where $\nu(u)=f(\bsx_{t,u}{:}\bsx_{b,-u})$.
Here $\bsx_t$ is the target point and we seek to explain
the difference between $f(\bsx_t)$ and $f(\bsx_b)$
where $\bsx_b\in\cx$ is some baseline or default point.
It could be a real point or the average of all $n$ points.
Baseline Shapley is considered an `interventional' method
because it is implemented by intervening to change some
of the components of $\bsx_b$ to those of $\bsx_t$.

Another measure is conditional expectation Shapley
where $\nu(u) = \e(f(\bsx_t)\giv \bsx_{tu})$.
More precisely, conditional expectation Shapley is
a family of measures because multiple
choices for the conditional distribution of $\bsx_t$
given $\bsx_{tu}$ have been considered.
Random Baseline Shapley methods \citep{sundararajan2020many, lundberg2017unified}
assume that the components of $\bsx_{t,-u}$ are drawn
independently from the marginal distributions of $\bsx$,
so they are generally off-manifold.
\cite{aas2021explaining} estimate the conditional distribution using kernel
methods.

One difficulty with baseline Shapley and other interventional methods
is that they can use extremely unlikely or even physically impossible variable
combinations.  If the point $\bsx_{t,u}{:}\bsx_{b,-u}$ is
completely impossible then the value $f(\bsx_{t,u}{:}\bsx_{b,-u})$
may be meaningless and hence not fit for use.
The cohort Shapley method of \cite{mase2019explaining} was constructed
to avoid using impossible combinations.  It does so by only using
actually observed data.  It begins with user-specified notions of
whether the value $x_{ij}$ is similar to $x_{tj}$.  For binary features
or those with a small number of levels it is natural to take
similarity to mean that $x_{ij}=x_{tj}$.  For continuously distributed
variables, one can take similarity to mean $|x_{ij}-x_{tj}|\le\delta_j$
or one can discretize $\cx_j$ into a modest number of buckets.
It must always be true that $x_{tj}$ is similar to itself.
If all $d$ variables have similarity defined in terms of equality,
perhaps after discretization, then CS becomes conditional expectation
Shapley based on the empirical distribution of $\bsx_i$.
The requirement to define similarity is a burden.
However, alternatives that use a black box to estimate
conditional expectations are less transparent in that
the effective definition of similarity is obscured.

We let $S_j(\bsx_i)$ be $1$ if $x_{ij}$ is similar to $x_{tj}$ and
and $0$ otherwise.
For $u\subseteq[d]$, we let $S_u(\bsx_i)=\prod_{j\in u}S_j(\bsx_i)$
with $S_\emptyset=1$ by convention.
The value function for CS is
$$
    \nu(u)=\nu(u;t) = \frac{\sum_{i=1}^n f(\bsx_i)S_u(\bsx_i)}{\sum_{i=1}^nS_u(\bsx_i)}.
$$
It is the average value of $f(\bsx_i)$ over the cohort
$C_u:=\{i\in[n]\mid S_u(\bsx_i)=1\}$ of observations similar to the
target $\bsx_t$ for all $j\in u$.  For $\bsx_i\in C_u$
we might or might not have $S_j(\bsx_i)=1$ for some $j\not\in u$.
The cohort mean $\nu(u)$ is always well defined because
$C_u$ is never empty, as $t\in C_u$.  With this definition, the Shapley
value explains
\begin{align}\label{eq:whatcsexplains}
    \frac1{C_{[d]}}\sum_{i\in C_{[d]}}f(\bsx_i)-\frac1n\sum_{i=1}^nf(\bsx_i).
\end{align}
The mean on the left of~\eqref{eq:whatcsexplains}
is the average of $f(\bsx_i)$ over observations $i$
in the `fully refined cohort' $C_{[d]}$ of observations similar to $\bsx_t$
on all $d$ variables.  Very commonly $C_{[d]}=\{t\}$.  Then CS attributes
the difference $f(\bsx_t)-(1/n)\sum_{i=1}^nf(\bsx_i)$ to the $d$
individual input variables through the Shapley values $\phi_j$ that
correspond to $\nu$.

The set of Shapley values $\phi_1,\dots,\phi_d$ depend on
$d2^{d-1}$ incremental values.  For special constructions of $f$
there are faster algorithms, with TreeSHAP \citep{lundberg2020local} being a well
known special case, but for model-agnostic Shapley values, we
cannot avoid the exponential cost.

\subsection{Integrated Gradients}

Integrated gradient (IG) methods are an alternative variable
importance measure that does not have exponential cost.
This approach has its own set of axioms which have
some similarities and some differences when compared
to the Shapley axioms.
Our interest is centered on using IG to compute an approximate
Shapley value so we don't list the IG axioms.
Those axioms are given in \cite{friedman2004paths}
for game theoretic coverage.
See \cite{sun2011axiomatic} and \cite{sundararajan2017axiomatic}
for variable-importance contexts.

Suppose that we have a baseline point $\bsx_b\in\real^d$
and a target point $\bsx_t\in\real^d$ and we want to
explain the difference $\tilde f(\bsx_t)-\tilde f(\bsx_b)$, through
an additive attribution to the components of the input points.
Here $\tilde f$ is a black box prediction function that
we will soon transform to a more convenient function denoted by $f$.
The contexts in \cite{sundararajan2017axiomatic}
include object recognition in images
of $d$ pixels and natural language settings where
there are $d$ words. In such cases, $d$ is large enough
that the exponential cost
of Shapley value forces one to approximate it.
The baseline $\bsx_b$ might be a null image (all black or
all white) or it might be a null document (all word
counts are zero).

Let $\tilde f$ be differentiable with gradient
$$
    \nabla \tilde f(\bsx) =
    \begin{pmatrix}
        \frac\partial{\partial x_1}\tilde f(\bsx),\dots,
        \frac\partial{\partial x_d}\tilde f(\bsx)
    \end{pmatrix}
$$
at any point $\bsx$ where
$\min(x_{bj},x_{tj})\le x_j\le\max(x_{bj},x_{tj})$
holds for all $j\in[d]$.  Then the integrated gradient
importance measure for variable $j$ is defined to be
$$
    \psi_j = (x_{tj}-x_{bj})\int_0^1\nabla \tilde f(\bsx_b+\alpha(\bsx_t-\bsx_b))\rd \alpha
$$
for all $j\in[d]$.
In game theory, this is known as the Aumann-Shapley
index after \cite{auma:shap:1974}.
We readily find by the fundamental theorem of calculus that
$\tilde f(\bsx_t)-f(\bsx_b)=\sum_{j=1}^d\psi_j$,
so IG satisfies the efficiency axiom of Shapley value.
In practice the integral can be estimated by a Riemann
sum.

For our purposes it is convenient to introduce
$f(\bsz) = \tilde f(\bsx_b+\bsz(\bsx_t-\bsx_b))$
for points $\bsz\in[0,1]^d$. Now our function of
interest is defined on the unit cube with
baseline point $\bszero$ and target point $\bsone$
at opposite corners of $[0,1]^d$.  The integrated gradient measures $\psi_j$
for $f$ are the same as the ones for $\tilde f$,
by the chain rule, so
$$
    \psi_j = \int_0^1 \nabla f(\alpha\bsone)\rd\alpha,\quad j\in[d].
$$
We integrate the gradient of $f$ from $\bszero$
to $\bsone$ along a path given by the main diagonal
in the unit cube.  IG is then a path integral method. See
\cite{friedman2004paths} and \cite{sundararajan2017axiomatic}
for more general path integral methods.

Now suppose that we have a value function
$\nu(u) = \tilde f(\bsx_{t,u}{:}\bsx_{b,-u})$.
This is commonly called baseline Shapley \citep{sundararajan2020many}.
Then $\nu(u)=f(\bsone_u{:}\bszero_{-u})$
and the Shapley values $\phi_j$ derived from $\nu(\cdot)$
depend only on the values that $f$ takes on the
set $\{0,1\}^d$ of corners of the cube $[0,1]^d$.

It is known that $\psi_j=\phi_j$ (the Shapley value)
under certain conditions on $f$.
\cite{sun2011axiomatic} note that
this match happens for functions $f:[0,1]^d\to\real$
that are multilinear such as $\prod_{j\in u}z_j$ or
functions such as $\prod_{j\in u}z_j\prod_{j\not\in u}(1-z_j)$
that can be expressed as sums of multilinear functions
or smooth additive
functions $\sum_{j=1}^dg_j(x_j)$.

We show next that the agreement generalizes beyond those above
mentioned cases.
We have not seen this more general agreement in the literature.
While it may be previously known in game theory, it
does not appear in \cite{owen1972multilinear}
and was not known in the XAI literature as of \cite{sundararajan2017axiomatic}.
Our argument uses each Shapley axiom once.
\begin{theorem}\label{thm:prodmatch}
    Let $h:[0,1]\to\real$ be a differentiable function
    and for nonempty $u\subseteq[d]$ let $f(\bsz)=\prod_{j\in u}h(z_j)$.
    Then integrated gradients $\psi_j$ for $f(\bsz)$ on $[0,1]^d$
    match the Shapley values $\phi_j$ for $\nu(u)=f(\bsone_u{:}\bszero_{-u})$.
\end{theorem}
\begin{proof}
    From the dummy axiom, $\phi_j$ is zero for $j\not\in u$.
    From the symmetry axiom $\phi_j$ for $j\in u$ are all equal.
    From the efficiency axiom they must sum to $h(1)^{|u|}-h(0)^{|u|}$.
    Therefore
    $\phi_j=1_{j\in u}(h(1)^{|u|}-h(0)^{|u|})/|u|$.
    For $j\in u$
    $$
        \psi_j = \int_0^1 h'(z)h(z)^{|u|-1}\rd z
        =\Bigl[\frac{h(z)^{|u|}}{|u|}\Bigr]^1_0=\frac{h(1)^{|u|}-h(0)^{|u|}}{|u|}=\phi_j
    $$
    while for $j\not\in u$ the $j$'th component of $\nabla f$ is zero,
    and then $\psi_j=0=\phi_j$.
\end{proof}

\begin{remark}
    The function $h$ does not have to be differentiable.  It is
    enough for $h^{|u|}$ to be absolutely continuous on $[0,1]$.
    Absolute continuity here means that $h^{|u|}$ has a
    Lebesgue integrable anti-derivative
    $g$ with $\int_0^ag(\alpha)\rd\alpha = h(a)^{|u|}$ for all $a\in[0,1]$.
    Then $h^{|u|}$ would be differentiable almost everywhere in $[0,1]$.
    For instance, if $h=\max(0,x-c)$ for some $c\in(0,1)$ then
    $h^{|u|}$ is absolutely continuous for all $|u|\ge1$
    but is only differentiable for $|u|\ge2$.  More generally,
    $h$ could be any continuous piece-wise linear function.
    Without a gradient, we would estimate $\psi_j$ by
    $$
        \sum_{r=0}^{R-1} f\Bigl( \frac{r}{R}\bsone + \frac1R
        e_j\Bigr)-f\Bigl( \frac{r}R\bsone\Bigr)
    $$
    for some large integer $R$, where $e_j$ is the $j$'th
    coordinate vector, $\bsone_j{:}\bszero_{-j}$.
\end{remark}

\begin{corollary}\label{cor:generalmatch}
    Let
    $$
        f(\bsz) = \sum_{\emptyset\ne u\subseteq [d]}\sum_{\ell=1}^{L_u}\prod_{j\in u}
        h_{u,\ell}(z_j)
    $$
    for arbitrary differentiable functions $h_{u,\ell}$ on $[0,1]$.
    Then the integrated gradients $\psi_j$ for $f(\bsz)$ on $[0,1]^d$
    match the Shapley values $\phi_j$ for $\nu(u)=f(\bsone_u{:}\bszero_{-u})$.
\end{corollary}
\begin{proof}
    The result follows from Theorem~\ref{thm:prodmatch} and the
    additivity axiom of Shapley value.
\end{proof}
We can of course add a constant term ($u=\emptyset$) to the function $f(\bsz)$ above.
For $|u|=1$ we do not need $L_u>1$, but for $|u|>1$ allowing
$L_u>1$ provides a generalization.

\begin{remark}
    We see from Corollary~\ref{cor:generalmatch} how $f$
    can contain a general smooth additive component (as
    was well known).
    For instance, if $f(\bsx)=\sum_{j=1}^dg_j(x_j)$ for
    smooth functions $g_j$,
    then by taking $h_{\{j\},1}=g_j$
    we see that IG matches Shapley value.  Interactions are
    more constrained.  For two-factor interactions, bilinear
    forms like $x_1x_2$ or $(x_1-c_1)(x_2-c_2))$ for constants $c_j$
    were already known to have IG match Shapley. More general quantities,
    for instance $x_1x_2+\exp(x_1)\exp(x_2)+\cos(x_1)\cos(x_2)$
    or sums of products of piece-wise linear continuous
    functions are now included in Corollary~\ref{cor:generalmatch}
    by defining the corresponding $h$'s appropriately.  This
    provides some increased generality.  All such interaction
    functions covered by the sufficient condition
    in Corollary~\ref{cor:generalmatch}
    are invariant under permutations of their arguments,
    so they are still quite special.
\end{remark}

\subsection{Integrated Gradients for Cohort Shapley}

Here we construct an IG method for CS.
Let $\nu(u)$ be the value function in CS.
We will take some notational liberty in
extending the domain of $\nu$ from sets
$u\subseteq[d]$ to corresponding points $\bsz\in\{0,1\}^d$
and then from there to a differentiable
function of $\bsz\in[0,1]^d$.  Then we apply
IG to this extended $\nu$.
The first step is to write $\nu(\bsone_u{:}\bszero_{-u})=\nu(u)$.
Then the cohort Shapley value function for the set $u$
is placed at the corner of $\{0,1\}^d$ that has $z_j=1$
for $j\in u$ and $z_j=0$ for $j\not\in u$.

To get an approximation to CS that scales to high dimensional
settings, we introduce a soft similarity function
$$
    s_{\bsz}(\bsx_i)=s_{\bsz}(\bsx_i;t) = \prod_{j=1}^d\bigl(
    1+z_j(S_j(\bsx_i)-1)\bigr)
$$
for $\bsz\in[0,1]^d$. Here $1+z_j(S_j(\bsx_i)-1)$ is
a soft version of $S_j(\bsx_i)$ that linearly interpolates from
$1$ to $S_j(\bsx_i)\in\{0,1\}$ as  $z_j$ increases from $0$ to $1$.
Soft similarity satisfies $s_{\bsone_u{:}\bszero_{-u}}(\bsx_i)=S_u(\bsx_i)$.
We can now extend $\nu$ from $\bsz\in\{0,1\}^d$ to $\bsz \in [0, 1]^d$ via
\begin{align}\label{eq:average_by_sz}
    \nu(\bsz) = \left. \sum^n_{i=1}f(\bsx_i)s_{\bsz}(\bsx_i) \Bigm/ \sum^n_{i=1}s_{\bsz}(\bsx_i).\right.
\end{align}
We have used $\nu(\cdot)$ instead of $f(\cdot)$
for our value function so that we can use
$f(\cdot)$ to denote the value of the
prediction at a data point.  In IG those two functions
can be the same, but for CSIG they must be different.

We interpret the numerator of $\nu(\bsz)$ as a soft total
and the denominator as a soft cardinality.
Note that while the components $z_j$ are real values in $[0,1]$,
the data features $x_j\in\cx$ do not have to be real-valued.
As a result, the soft value function $\nu $ in~\eqref{eq:average_by_sz}
can be defined for very general features.
The quantity explained is
$\nu(\bsone)-\nu(\bszero)$ which reduces to the
fully refined cohort mean minus the sample mean
as in CS.  The cohort Shapley value function for the set
$u$ equals the soft cohort value $\nu(\bsone_u{:}\bszero_{-u})$.

The integrated gradient cohort Shapley (IGCS) values are
given by
$$
    (\psi_1,\dots,\psi_d) = \int_0^1\nabla\nu(\alpha\bsone)\rd \alpha
    \quad\text{where}\quad\nabla\nu(\bsz)=\begin{pmatrix}
        \frac\partial{\partial z_1}\nu(\bsz),\dots,\frac\partial{\partial z_d}\nu(\bsz)
    \end{pmatrix}.
$$
The partial derivatives we need depend on
$$
    s_{\bsz}^{(k)}(\bsx_i) \equiv \frac{\partial}{\partial z_k} s_{\bsz}(\bsx_i) = (S_k(\bsx_{i})-1)\prod_{j\in -k} (1+z_j (S_j(\bsx_{i})-1)).
$$
Now
$$
    \frac{\partial}{\partial z_k} \nu(\bsz) = \frac{(\sum^n_{i=1}f(\bsx_i)s^{(k)}_{\bsz} (\bsx_i))(\sum^n_{i=1} s_{\bsz}(\bsx_i))-(\sum^n_{i=1}f(\bsx_i)s_{\bsz} (\bsx_i))(\sum^n_{i=1} s^{(k)}_{\bsz}(\bsx_i))}{(\sum^n_{i=1} s_{\bsz}(\bsx_i))^2}.
$$
This partial derivative can be computed at cost $\mathcal{O}(nd)$.
Integrating it with a quadrature rule on $R$ nodes then costs $\mathcal{O}(nRd)$.
Now suppose that we evaluate the gradient at a point $\alpha\bsone$. We get
\begin{align*}
    s_{\alpha\bsone}(\bsx_i)       & = \prod^d_{j=1}(1+\alpha(S_j(\bsx_{i})-1)),\quad\text{and}      \\
    s_{\alpha\bsone}^{(k)}(\bsx_i) & = (S_k(\bsx_{i})-1) \prod_{j\in -k}(1+\alpha(S_j(\bsx_{i})-1)).
\end{align*}
The components of the gradient are rational functions in $\alpha$ with a numerator degree of $2d - 1$ and a denominator degree of $2d$.
We integrate them via an equally weighted
average over equispaced points $\alpha_r\in[0,1]$
for $r=1,\dots,R$.

The variable importance interpretation of $\psi_j$
for CSIG is similar to that of $\phi_j$ for CS.
In CS, we remove observations with dissimilar
values of $x_{ij}$ from the cohort when
we refine on variable $j$, and take account of
the resulting changes in the cohort mean.  In CSIG we
continuously downweight such observations as $\alpha$
increases from $0$ to $1$, and track the resulting
changes in the soft cohort mean.  In both cases, importance
derives from aggregated changes to the corresponding mean
as a distribution on $\bsx_i$ concentrates more nearly
to $\bsx_t$.

\section{CS Versus IGCS for Large $d$}
\label{sec:asymptote}

In the CERN example of Section~\ref{sec:experiment},
with $d=16$, we will see that cohort Shapley and integrated gradient
cohort Shapley attain very nearly the same
performance.  That is quantified by an
area between the curves  (ABC) quantity
derived from the area under the curve (AUC)
quantity of \cite{petsiuk2018rise}. See Appendix~\ref{sec:abc}.
We know that IGCS cannot always
match CS because CS requires exponential computation
and IGCS does not.
In this section we explain why CS and
IGCS can be expected to be very close to each
other in high dimensional settings
where CS cannot be computed, such as a $1024$-dimensional
example from computational chemistry, also in Section~\ref{sec:experiment}.

While both the numerator and denominator of $\nu(\bsz)$
are multilinear functions where the integrated gradient recovers
the Shapley values,
this is not true of the ratio itself.  We use a Taylor expansion
of the soft cohort mean around the point $\bsone\in[0,1]^d$
to explain how the ratio is very nearly multilinear
over most, though not all, of $[0,1]^d$.

It helps to introduce dissimilarity sets
$$
    J_i = \bigl\{j\in[d]\mid S_j(\bsx_i)=0\bigr\}.
$$
Then the soft similarity for observation $i$ is
$$
    s_{\bsz}(\bsx_i) = \prod_{j\in J_i}(1-z_j).
$$
Now we make a formal Taylor series expansion
\begin{align}\label{eq:formaltaylor}
    \nu(\bsz) & = \Bigl(
    f(\bsx_t)+\sum_{i\ne t}f(\bsx_i)\prod_{j\in J_i}(1-z_j)
    \Bigr)\Bigm/
    \Bigl(1+\sum_{i\ne t}s_{\bsz}(\bsx_i)\Bigr)\notag \\
              & =
    \Bigl(
    f(\bsx_t)+\sum_{i\ne t}f(\bsx_i)\prod_{j\in J_i}(1-z_j)
    \Bigr)
    \Bigl( 1+\sum_{r=1}^\infty
    \Bigl(-\sum_{i\ne t}s_{\bsz}(\bsx_i)\Bigr)^r
    \Bigr).
\end{align}
This formal expansion fails to converge
for some values of $\bsz$.  We next give conditions
under which it converges for the vast majority
of $[0,1]^d$.  Those conditions
keep the soft cardinality (our denominator)
between $1$ and $1+\epsilon$ over
the vast majority of $[0,1]^d$. The lower bound is $1$ because
the target point $\bsx_t$ is always counted.

We assume that the fraction of variables $j\in[d]$ for
which $\bsx_i$ is dissimilar to $\bsx_t$
belongs to a sub-interval of $(0,1]$ as follows:
\begin{align}\label{eq:atoA}
    ad \le |J_i|\le Ad
\end{align}
for constants $0<a\le A\le1$.

\begin{theorem}\label{thm:itmostlyconverges}
    Let there be $n$ observations $\bsx_i$
    for $i\ne t$
    and suppose that equation~\eqref{eq:atoA} holds
    for the target point $\bsx_t$ and similarity
    functions $S_j(\bsx_i)$.
    For $0<\epsilon<1$ let
    \begin{align}\label{eq:defheps}
        H_\epsilon =
        \Bigl\{\bsz\in[0,1]^d\mid \sum_{i\ne t}\prod_{j\in J_i}(1-z_j)\ge\epsilon\Bigr\}.
    \end{align}
    Then for $\bsz\sim\dunif[0,1]^d$,
    $$
        \Pr(\bsz\in H_\epsilon)\le
        \frac{n^2}\epsilon\exp\Bigl(-\frac{\lfloor ad\rfloor}4\Bigr).
    $$
\end{theorem}

\begin{proof}
    We will use the fact that minus twice the
    sum of the logs of $m$ independent $\dunif(0,1)$
    random variables has the $\chi^2_{(m)}$ distribution.
    We get
    \begin{align*}
        \Pr( \bsz\in H_\epsilon)
         & \le\Pr\biggl( \max_{i\ne t}\prod_{j\in J_i}(1-z_j) \ge\frac\epsilon{n}\biggr)        \\
         & \le n\Pr\Biggl(\, \prod_{j=1}^{\lfloor ad\rfloor}(1-z_j)\ge
        \frac\epsilon{n}\Biggr)                                                                 \\
         & =n\Pr\biggl(-2\sum_{j=1}^{\lfloor ad\rfloor}\log(1-z_j)\le -2\log(\epsilon/n)\biggr) \\
         & = n\Pr( \chi^2_{(2\lfloor ad\rfloor)}\le 2\log(n/\epsilon)).
    \end{align*}
    Now equation (4.4) of \cite{laur:mass:2000} shows that
    for any integer $m\ge1$ and any $x>0$
    $$
        \Pr( \chi^2_{(m)} \le m - 2\sqrt{m x})\le e^{-x}.
    $$
    Taking $m=2\lfloor ad\rfloor$ and
    $x = (m-2\log(n/\epsilon))^2/(4m)$
    we get
    \begin{align*}
        \Pr(\bsz\in H_\epsilon)
         & \le n\exp\Bigl(-\frac{(m-2\log(n/\epsilon))^2}{4m}\Bigr) \\
         & \le n\exp\Bigl(-
        \frac{\lfloor ad\rfloor}4
        +\log\Bigl(\frac{n}\epsilon\Bigr)
        \Bigr).
    \end{align*}
    Moving $\log(n/\epsilon)$ out of the exponent completes the proof.
\end{proof}

As a result, we conclude that
in a setting with $d\gg \log(n)$, the Taylor
expansion~\eqref{eq:formaltaylor}
is convergent over all but a trivially
small part of the unit cube.
The theorem above is for $n$ points in addition
to the target point $\bsx_t$.  For $n-1$ such
points the bound is the somewhat less elegant
$(n-1)^2\exp(-\lfloor ad\rfloor/4)/\epsilon$.

\begin{remark}
    The bound in Theorem~\ref{thm:itmostlyconverges} does
    not involve the constant $A$.  This means that observations
    where $\bsx_i$ is dissimilar to $\bsx_t$ for most or even all
    features are not detrimental to the convergence of the
    Taylor series.
\end{remark}

\begin{remark}
    At the other extreme, it
    is possible that some point $\bsx_i$ is essentially
    a duplicate of $\bsx_t$ and is then similar to $\bsx_t$
    for all $d$ predictor variables.
    We write $i\sim t$ for that case and $i\not\sim t$
    otherwise.
    Suppose that there
    are $n_1\ge1$ points identical to $\bsx_t$ (including $\bsx_t$ itself) and  $n_2=n-n_1$ other points satisfying
    equation~\eqref{eq:atoA}.
    Then the soft cohort mean is
    \begin{align*}
         & \Bigl(\sum_{i:i\sim t}f(\bsx_i) +\sum_{i:i\not\sim t}f(\bsx_i)\prod_{j\in J_i}(1-z_j)\Bigr)\Bigm/
        \Bigl(n_1+\sum_{i:i\not\sim t}\prod_{j\in J_i}(1-z_j)\Bigr)                                          \\
         & =\Bigl(\bar f+\frac1{n_1}\sum_{i:i\not\sim t}f(\bsx_i)\prod_{j\in J_i}(1-z_j)\Bigr)\Bigm/
        \Bigl(1+\frac1{n_1}\sum_{i:i\not\sim t}\prod_{j\in J_i}(1-z_j)\Bigr)
    \end{align*}
    where $\bar f=\sum_{i:i\sim t}f(\bsx_i)/n_1$.
    The Taylor expansion in equation~\eqref{eq:formaltaylor}
    is convergent over the vast majority of $[0,1]^d$
    for this soft cohort mean too, by the same argument
    we used in Theorem~\ref{thm:itmostlyconverges}.
    This holds whether $n_1=O(1)$ or whether
    $n_1$ grows proportionally to $n$.
    We just need $d\gg \log(n_2)$.
\end{remark}

Now we turn to the corners of the cube
used in the definition of cohort Shapley
and show that the Taylor expansion is
convergent for the vast majority of them.
Clearly $\bszero\in H_\epsilon$ and $\bsone\not\in H_\epsilon$.
We let $\bsz=\bsone_u{:}\bszero_{-u}\in\{0,1\}^d$.

\begin{theorem}
    Let there be $n$ observations including $\bsx_t$
    and suppose that equation~\eqref{eq:atoA} holds
    for the target point $\bsx_t$ and similarity
    functions $S_j(\bsx_i)$.
    Then the Taylor
    expansion~\eqref{eq:formaltaylor} is convergent for at least
    $$2^d\Bigl(1-\frac{n}{2^{da}}\Bigr)$$
    of the points in
    $\{0,1\}^d.$
\end{theorem}
\begin{proof}
    For a corner point
    $\bsz=\bsone_u{:}\bszero_{-u}\in\{0,1\}^d$ to be in
    $H_\epsilon$ of equation~\eqref{eq:defheps}
    there must be at least one observation $i$
    with $\prod_{j\in J_i}(1-z_j)\ne0$.
    That is $u\subseteq J_i^c$.
    There are at most $2^{d-|J_i|}\le 2^{d(1-a)}$ such
    $u$.  As a result the Taylor expansion is convergent
    for at least $2^d-n2^{d(1-a)}=2^d(1-n/2^{da})$ corners of $[0,1]^d$.
\end{proof}

For any $\epsilon>0$ we have
$$1\le s_{\bsz}\le 1+\epsilon$$
with overwhelming probability for $\bsz\sim\dunif[0,1]^d$.
For such $\bsz$
\begin{align}\label{eq:taylor0}
    \nu(\bsz) \approx
    f(\bsx_t) + \sum_{i\ne t}f(\bsx_i)\prod_{j\in J_i}(1-z_j)
\end{align}
to within a factor of $1+\epsilon$
and the right hand side of~\eqref{eq:taylor0}
is a function for which IGCS matches CS.
Furthermore equation~\eqref{eq:taylor0} holds exactly
at the majority of the corners of $[0,1]^d$.

In this section, we have shown that the IGCS value
function $\nu$ defined on $[0,1]^d$ is one where, for large $d$,
we can expect IGCS to closely match CS.
More precisely, this $\nu$ is very nearly a multilinear
function on the majority of $[0,1]^d$ and at the majority
of corners $\{0,1\}^d$, and integrated gradients match Shapley
value for multilinear functions.  The discrepancy between CS and IGCS
is studied in Appendix~\ref{sec:higherorder},
which looks at the next term in the Taylor expansion.

\section{Experiments}
\label{sec:experiment}
In this section we show the feasibility of IGCS by applying it to real-world datasets.
We use a 16-dimensional example from high energy
physics where we can compare CS to IGCS.
Then we use a 1024-dimensional example
from computational chemistry where CS is
infeasible.

\subsection{Comparison of IGCS to CS for CERN Data}
\label{sec:cern}
Here we compare the performance of IGCS to the Cohort Shapley \citep{mase2019explaining} with a low dimensional dataset.
We evaluate the XAI methods by the insertion and deletion tests \citep{petsiuk2018rise, hama2022deletion}
described in Section~\ref{sec:abc}.
We will consider both interventional and conditional
ABC measures.  Methods that are better at ordering
variables by importance get larger ABC values.


\subsubsection{Dataset and Setup}
\label{sec:cernexp}
The CERN Electron Collision Data
\citep{cerndataset}
is a dataset about dielectron collision events at CERN.
It includes continuous variables
representing the momenta and energy of the electrons,
as well as discrete variables for the charges of the electrons ($\pm 1$: positrons or electrons).
Only the data whose invariant mass of two
electrons (or positrons) is between $2$ and $110$ GeV are included.
We treat it as a regression problem to predict
their invariant mass from the other 16 features.
Our model for this problem
is a neural network described in
Appendix~\ref{sec:cerncolldata}.

To apply CS we must choose similarity functions $S_j(\bsx_i)$.
Define the range of feature $j$ to be $r_j=\max_ix_{ij}-\min_ix_{ij}$.
We take $S_j(\bsx_i)$ to be one if
and only if $|x_{ij}-x_{tj}|\le \delta_jr_j$.
We set each $\delta_j=0.1$ so that
two values of variable $j$ are similar
if they are within 10\% of the range of the variable.

\subsubsection{Comparison Methods}
As noted in Section~\ref{sec:rel}, the
supervised on-manifold method of \cite{frye2020shapley} is a
whole framework
of methods, and it has exponential cost.
It is not clear how to develop
one for this specific application, or how to
train it well as the model it uses combines
$2^{1024}$ different submodels depending on
what variables are missing.
For the purposes of comparison we have modified
the Gaussian kernel weight (GKW) method from
\cite{aas2021explaining} to get a model-free
method.
This modified method
is described in detail in Appendix~\ref{app:egkw}.
It was necessary to make a modification so
that the method could be computed using only
observed values of $f(\bsx_i)$ instead
of unobserved hybrid points.

The uniqueness Shapley criterion of \cite{seiler2021makes} is a Shapley
value based on taking $\nu(u)$ to be $-\log_2$ of the
cardinality of the cohort $C_u$.  Then the Shapley values $\phi_j$
quantify the relative importance of the variables in
separating observation $t$ from the others.  When aggregated
over all observations it becomes a weighted sum of conditional
entropies.  We include it as a comparison because it is
computationally feasible and it allows us to quantify whether
variables for which a similarity condition
more strongly makes the observation unique might also
be important for predicting model output.  That is, a
variable in which $\bsx_t$ tends to be an outlier
might, to some extent, also be a variable that strongly
affects its prediction or its response or
its residual.

The prediction function for the CERN data
is a differentiable function of $16$
continuously varying variables.
That makes it possible to use the original
integrated gradients method on the model
output (but not on the residuals).

The IG method compares a target data point to a baseline point.
This is different from CS which ranks variables by how
they split the point's response value from the set
of all observations' responses.
For IG we took all of the $2000$ points in turn as
the target point $t$ and computed the IG
values from it to the other $1999$ points.
All in all we computed ${2000\choose 2}$
insertion and deletion ABCs.
We do not compute ABCs for kernel SHAP (KS)
because the cost of doing ${2000 \choose 2}$
KS computations is prohibitive.  We have extensive experience
running KS on this same data set \citep{hama2022deletion}.
Based on that experience we expect KS to perform better than IG,
but not much better than IG which was always a close
second to KS in those comparisons.

We also include random variable ordering.
It is known \citep{hama2022deletion}
that under random ordering
the expected value of ABC for either
insertion or deletion is zero when $f$
is additive. Whether $f$ is additive or not,
the expected sum of insertion and deletion
ABCs is zero.
The random ordering used was different for each target point.

\subsubsection{Results for CERN Data}
Table~\ref{tb:diffABC_cern_predicted} shows insertion
and deletion ABCs for the prediction function $f$
on the CERN data. These are conditional ABCs
as described in Section~\ref{sec:abc}.
We include some standard errors, with
respect to the
$2000$ held out points, as a descriptive statistic.
For IG it is a standard error of ${2000\choose 2}$
ABCs.
We see that CS works best but IGCS is close
behind even though there are only $16$ variables.
We also see that the sum of insertion and deletion
ABCs is very close to zero for random ordering,
consistent with theory.  The empirical GKW
method gets just over half the ABC values that CS and
IGCS get.  It outperforms plain IG.  That could be
because IG is an interventional method and
Table~\ref{tb:diffABC_cern_predicted} reflects conditional
(cohort) ABCs.  Uniqueness Shapley gets a low score.
However, that measure does not even use the $f(\bsx_i)$
values and so it is interesting that it can partially identify
variables which move the cohort mean using only
a measure of how outlying the target point is in those variables.

\begin{table}[t]
    \centering
    \begin{tabular}{llcc}
        \toprule
        Test Mode & Method               & Mean           & Std. Error \\
        \midrule
        Insertion & Cohort Shapley       & 10.213$\phz$   & 0.142      \\
                  & IGCS                 & 9.726          & 0.137      \\
                  & Random               & $-$0.461$\phm$ & 0.130      \\
                  & empirical GKW        & 6.569          & 0.139      \\
                  & Uniqueness Shapley   & 1.243          & 0.178      \\
                  & Integrated Gradients & 2.260          & 0.004      \\
        \midrule
        Deletion  & Cohort Shapley       & 9.176          & 0.114      \\
                  & IGCS                 & 8.835          & 0.122      \\
                  & Random               & 0.406          & 0.129      \\
                  & empirical GKW        & 5.866          & 0.106      \\
                  & Uniqueness Shapley   & 3.117          & 0.164      \\
                  & Integrated Gradients & 2.192          & 0.003      \\
        \bottomrule
    \end{tabular}
    \caption{Mean insertion and deletion ABCs for $2000$ of the CERN data points calculated on feature attribution
        to predicted values, rounded to three places.
        The standard errors are with respect to $2000$
        held-out points for all methods except IG
        where they are for ${2000\choose 2}$ pairs of points.
    }
    \label{tb:diffABC_cern_predicted}
\end{table}

Table~\ref{tb:diffABC_cern_residue} shows insertion and deletion
ABCs for the residuals in the model. The ABCs are smaller
than for the prediction itself because the model has fit well.
Again we see highest ABCs for CS with IGCS a close second
and empirical GKW at about half the ABC.  It is not possible to
use the original IG method here because the residuals are not
a differentiable function of the predictors.

\begin{table}[t]
    \centering
    \begin{tabular}{llcc}
        \toprule
        Test Mode & Method             & Mean           & Std. Error \\
        \midrule
        Insertion & Cohort Shapley     & 1.239          & 0.022      \\
                  & IGCS               & 1.155          & 0.021      \\
                  & Random             & $-$0.023$\phm$ & 0.017      \\
                  & empirical GKW      & 0.601          & 0.019      \\
                  & Uniqueness Shapley & 0.108          & 0.023      \\
        \midrule
        Deletion  & Cohort Shapley     & 1.034          & 0.020      \\
                  & IGCS               & 0.993          & 0.020      \\
                  & Random             & 0.019          & 0.018      \\
                  & empirical GKW      & 0.527          & 0.017      \\
                  & Uniqueness Shapley & 0.254          & 0.023      \\
        \bottomrule
    \end{tabular}
    \caption{Mean insertion and deletion ABCs for $2000$ of the CERN data points calculated on feature attribution
        to residuals, the difference between annotated values and predicted values.}
    \label{tb:diffABC_cern_residue}
\end{table}

Our conditional ABCs can be considered unfair to
interventional methods like IG because those
methods favor variables
that bring big changes in interventions but were scored by
how much they moved the cohort mean.
To address this, we computed some interventional ABCs
and we report them in Table \ref{tb:diffABC_cern_interventional}.
For an interventional method we need to choose a set
of baseline-target pairs $(\bsx_b,\bsx_t)$. This is different
from our conditional method comparing $\bsx_t$ to some
average of all data.
The pairs we select are those from
the counterfactual policy in \cite{hama2022deletion}.
That policy chooses the pairs as follows. For a given target $\bsx_t$
we select baselines $\bsx_b$ with:
\begin{compactitem}
    \item $x_{bj}\ne x_{tj}$ for all $j\in[d]$ including both particles' charges,
    \item $\bsx_b$ has one of the $20$ smallest such
    $\Vert\bsx_i-\bsx_t\Vert$ values, and
    \item it maximizes $|f(\bsx_b)-f(\bsx_t)|$ subject to the above.
\end{compactitem}
The rationale for that policy is given in \cite{hama2022deletion}.
Pairs of responses can differ by more than responses differ from
the average and we have also selected pairs that differ greatly
(giving something to explain) so it is not surprising that larger ABCs can
be found for this counterfactual policy.
Because we adopt this policy also for picking reference data in calculation of
IG and KS,
we are able to compute KS in this setting.
It performed best in \cite{hama2022deletion}. IG is
nearly as good.  Conditional methods like CS and IGCS do not score
well on these interventional ABCs. The empirical GKW
method is the best conditional method under this interventional
scorecard.

\begin{table}
    \centering
    \begin{tabular}{llcc}
        \toprule
        Test Mode & Method               & Mean         & Std. Error \\
        \midrule
        Insertion & Cohort Shapley       & 5.807        & 0.169      \\
                  & IGCS                 & 6.077        & 0.171      \\
                  & empirical GKW        & 8.759        & 0.179      \\    &Kernel SHAP& 18.535$\phz$ & 0.215\\
                  & Integrated Gradients & 18.289$\phz$ & 0.213      \\

        \midrule
        Deletion  & Cohort Shapley       & 5.884        & 0.176      \\
                  & IGCS                 & 6.176        & 0.175      \\
                  & empirical GKW        & 9.259        & 0.156      \\
                  & Kernel SHAP          & 16.752$\phz$ & 0.176      \\
                  & Integrated Gradients & 16.315$\phz$ & 0.173      \\
        \bottomrule
    \end{tabular}
    \caption{Mean insertion and deletion ABCs evaluated in interventional measures for $2000$ of the CERN data points calculated on feature attribution
        to predicted values, rounded to three places.
        The scores in other XAI methods are reported in \cite{hama2022deletion}.}
    \label{tb:diffABC_cern_interventional}
\end{table}

\subsection{Feasibility in Data with Many Features}\label{sec:logp_exp}
In this section we show the feasibility of IGCS in a dataset
with 1024 binary features from an experiment in chemoinformatics.
Because the features are binary, it is not difficult
to define similarity. We take $S_j(\bsx_i)=1$ if and
only if $x_{ij}=x_{tj}$.
We cannot compute CS exactly for this problem,
but we can relate the IGCS
findings to some domain knowledge.
We can compute CS by Monte Carlo. We find that the resulting  estimates do not attain
ABC values competitive with IGCS.
We cannot compute IG because the variables are all
binary.

\subsubsection{Dataset and Setup}
For one elementary task in chemoinformatics,
we consider the estimation of ``logP'' that
measures the preference of a molecule to dissolve in lipids
versus water.
This metric is one of the rough standards in medicine
known as Lipinski's rule of five where it is an aspect
of `druglikeness'.
We use DeepChem \citep{Ramsundar-et-al-2019} and RDKit
\url{https://www.rdkit.org/}, a framework for chemoinformatics,
and the data are collected from ZINC15 \citep{sterling2015zinc}.
The response we use are annotated values
that are not experimental values but
are instead calculated as sums of contributions from
fragments of molecules called Wildman-Crippen LogP values \citep{wildman1999prediction}.

The input data for these models
encodes the molecules via the Extended-Connectivity Fingerprints (ECFP) \citep{rogers2010extended}.
Each fingerprint is
a binary vector that represents local information around each
atom in a given molecule.
This local information is transformed by a hash function and stored as a 1024-dimensional binary vector.
Each of these 1024 bits represents whether one or more copies
of the corresponding information is present (bit=$1$) or whether
none exist at all (bit=$0$), i.e., the feature is absent.

The 1024 features we study were derived from
a hash of 45{,}033 original features.
This means that there are hash collisions
of about 44 original features for each bit
in our data.

Before presenting our model, we describe
two papers doing related work.  They are
both about using XAI on similar problems
to the one we discuss.

When two structurally similar compounds have
very different chemoinformatic properties,
it is called an activity cliff.

\cite{tamura2021interpretation} considered blackbox prediction
of activity cliffs.
Their feature set is somewhat different
from the one we use.
ECFP4 is a variant of ECFP that focusses
on nonlocal information.
They concatenate some binary
vectors from ECFP4 to get the features
they use.
They had more than 2000 features in all.
They applied kernel SHAP, presumably with
Monte Carlo sampling given the high dimensional
feature vectors, to their support vector
machine model.
In their application, KS was unsatisfactory.
Because it attributed a lot of importance to absence
of features, the domain experts could
not use the results.
They found better (more interpretable)
results with the Tanimoto
index computed with variables derived from domain knowledge.

\cite{heberle2022xsmiles} compared Crippen values and results of TreeSHAP.
They used the Morgan Fingerprint, which is comparable to ECFP, as input vectors.
They confirmed that the TreeSHAP is ``almost identical'' to Crippen contributions ``relative to their own maximum absolute value'', that is, after normalizing by the maximum attribution.
For this reason, we also compare the IGCS values to assigned Crippen values as a very rough guideline in following.

The model we used to estimate the annotated logP values
is described in Appendix~\ref{sec:logp}.  This model
is very shallow,
containing just one hidden layer with 3935 neurons
to avoid overfitting.
It was fit using mean squared error loss.
The data set we consider is made up of $2000$ molecules
from the test data, that is, molecules that were held out
during fitting of the one layer model.

\subsubsection{Individual Examples in Chemoinformatics}
We chose one molecule at random from the test data,
shown in Figure~\ref{fig:10_str}(a).  After we discuss that
molecule we will consider another molecule, the one for which the model
had its largest error, shown in Figure~\ref{fig:10_str}(b).
We call these ZINC10 and ZINC61 below, which are shortened versions of their IDs.  The pentagonal structures shown near the
left of each molecule are known as `aromatic rings'.  We will
refer to that structure below.
The best-known molecule with an aromatic ring is benzene, which has
a hexagonal aromatic ring composed of six Carbon atoms.

The randomly chosen molecule ZINC10 has annotated logP$=2.8777$.
The average logP over all 2000 molecules is $1.2194$
so we are left to explain a logP difference
of $1.6583$.
With IGCS we get $\sum_{j=1}^{1024}\hat\psi_j=1.6201$.
The discrepancy arises because of numerical quadrature
used to compute the integrated gradient
vector $\hat\psi$.
The predicted logP for this molecule was $2.555$
and we had an average prediction
of $1.2059$ over the $2000$ molecules.
This leaves $2.555-1.2059=1.3491$ to explain
and we got $\sum_{j=1}^{1024}\hat\psi_j=1.3033$.

\begin{figure}[t]
    \centering
    \includegraphics[width=.9\hsize]{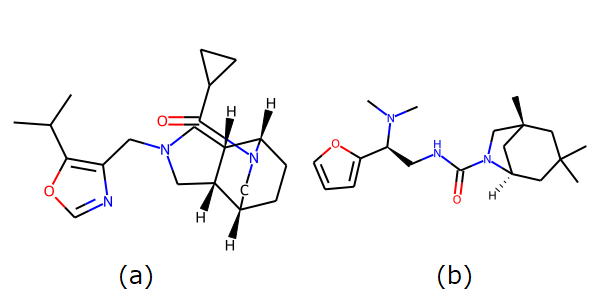}
    \caption{Two examples of structural formula in ZINC15 as
        drawn by RDKit.
        (a): A randomly chosen molecule
        whose ID in ZINC15 is ZINC1089752100.
        (b): The molecule with greatest absolute prediction error.
        The ID in ZINC15 is ZINC611639837.
    }
    \label{fig:10_str}
\end{figure}

The distribution of IGCS values is represented in
the left panel of Figure \ref{fig:10_61}.
The most negative $\hat\psi_j$
for this molecule was $-0.0279$, for the annotated value
and $-0.0310$ for the predicted value.
Both of those were for feature $726$ which was
absent (bit $=$ 0) for this molecule.
Feature 726
was present in $955$ of the $2000$ molecules
and absent in $1045$ of them.
Figure~\ref{fig:feat726}
shows IGCS values $\hat\psi_{726}(t)$ letting
the target $t$ be any of the $2000$ molecules.
We see that very generally $\hat\psi_{726}$ for
the true annotated values is quite close to that
for the predicted values for this feature.  Also,
presence of the feature is strongly linked to increasing
logP while absence is strongly linked to decreasing logP
with perfect separation of the signs for both annotated
and predicted values.
The molecule ZINC10 is shown as the green circle.
If bit 726 is present
it typically represents the existence of a carbon chain
of length three.  Such a chain is denoted by
ccc in
the Simplified Molecular-Input Line-Entry System (SMILES)
notation \citep{weininger1988smiles},
which we also use for some
other chemical features below.

\begin{figure}[t]
    \centering
    \includegraphics[width=\linewidth]{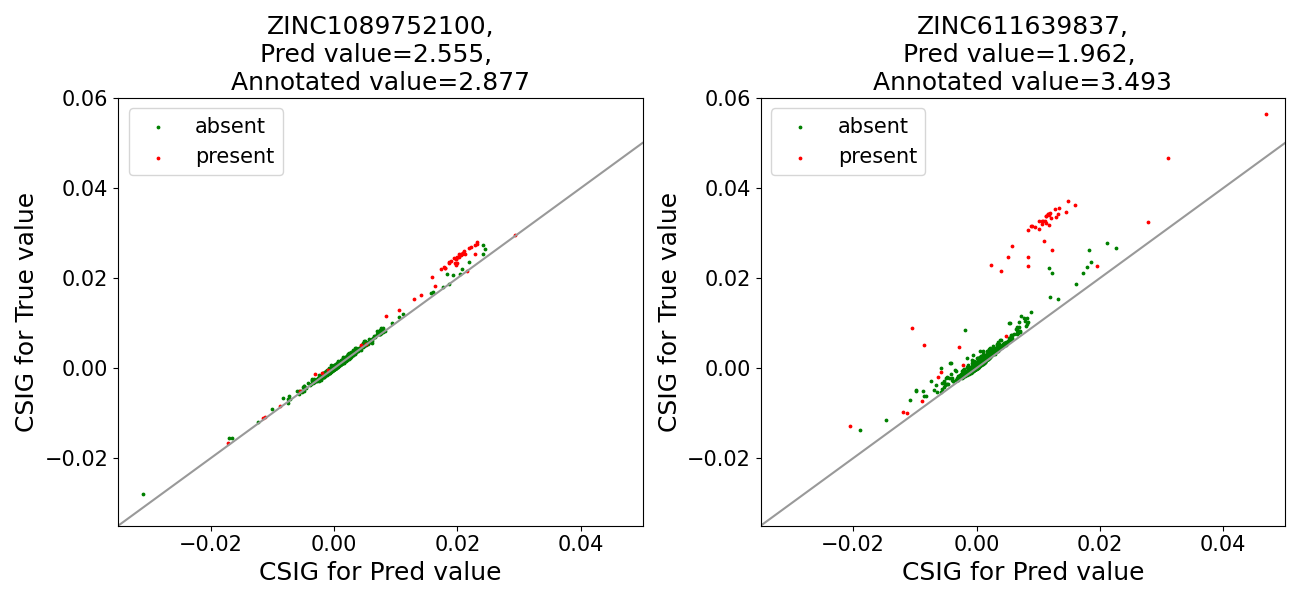}
    \caption{
        The contributions to the predicted value and the annotated value.
        Left: the molecule (a) in Figure \ref{fig:10_str}.
        Right: the molecule (b) in Figure \ref{fig:10_str}.        }
    \label{fig:10_61}
\end{figure}

\begin{figure}[t]
    \centering
    \includegraphics[width=0.5\linewidth]{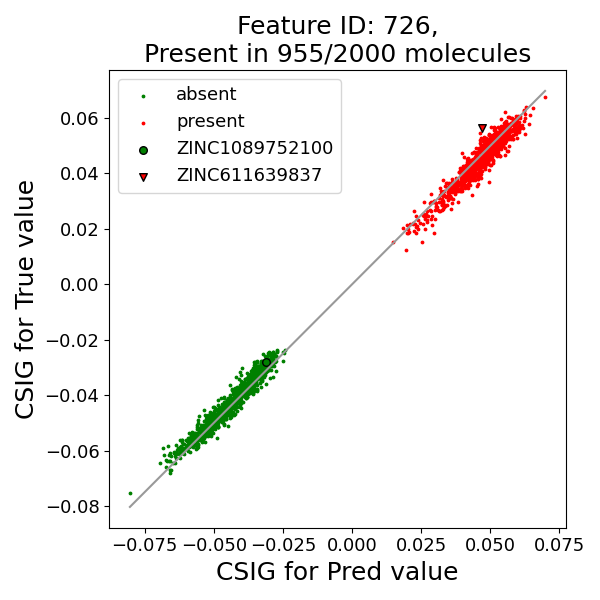}
    \caption{The distribution of IGCS assigned to feature 726 for all molecules in the dataset.
        The molecule of Figure \ref{fig:10_str}(a) is shown as a green circle,
        and the one from Figure \ref{fig:10_61}(b) is shown as a red triangle.
    }
    \label{fig:feat726}
\end{figure}

The second largest negative value
for ZINC10 is $-0.0166$ for the annotated value,
and $-0.0172$ for the predicted value.
This comes from feature $935$ that is present in ZINC10.
That feature represents the presence of Nitrogen
which is known to be a water-soluble part
which thus negatively contributes to logP.
The corresponding Crippen value is $-0.3187$,
and in this molecule there are two fragments of them
so the total contribution to the annotated value
is twice this Crippen value.
There are three $N$ in Figure \ref{fig:10_str}(a),
and the one in the aromatic ring contributes $-0.3239$
to the Crippen value. 

The largest positive $\hat\psi_j$
for ZINC10 is $0.0295$ for the annotated value
and $0.0293$ for the predicted values.
These come
from  feature $849$, which is present in ZINC10, representing a C--H part neighboring to the Nitrogen and the Oxygen in the aromatic ring.
The corresponding Crippen value is $0.1581$.

As shown above,
while the units of ECFP do not match to fragments where Crippen values are annotated and
they cannot be interpreted as ground truth of XAI,
the signs and rough magnitudes of Crippen values are good indicator of
domain knowledge in this field and our results agree with them.

The second molecule, shown in Figure~\ref{fig:10_str}(b)
has largest loss in the test dataset.
We denote that molecule `ZINC61' in this paper.
Before we study the
features that explain this discrepancy we remark on
the features that contribute to ZINC61's values.
The two largest positive contributions to annotated value $\hat\psi_j$
are from features $726$ and $64$, the same as for ZINC10.
They both come from the presence of carbon chains ccc in
the aromatic ring.
Because the neighboring information around
those carbon chains is different, they are
hashed into different features.
The largest negative feature
of IGCS for annotated values at 
ZINC61 is feature $875$
with a zero bit (absence).
Feature $875$ also typically
means the presence of another kind of carbon chain
in the aromatic ring.

The second largest negative feature in ZINC61 feature  $935$.
It has the largest negative IGCS value
among present features.
This feature also appeared in ZINC10 as the second largest negative feature.
It is about presence of N and
the assigned Crippen value for this atom is
$-0.3187$.

The feature where predicted importance falls
short of annotated importance by the greatest amount
is feature $210$.
This feature is present in this molecule.
Presence of feature $210$ is quite rare in the data set
because only 26 of the 2000 molecules have it present.
Presence is also rare for many other features.
While it looks like a very rare feature,
this dataset is highly sparse;
about $1/4$ of the features appear in $26$ or fewer molecules.
More precisely, feature 210 is the 251st `sparsest'
of the 1024 features.
Feature $210$ in this molecule
counts the presence of CCNC(N)=O (in SMILES notation).
The IGCS for annotated value is $0.03160$ and for predicted value is $0.00891$.
The second most underestimated feature is number
$955$. It refers the presence of CNC[C@@H](c)N in this molecule.
The number of molecules whose bit $955$ is one is 22,
but only 4 of them corresponds to CNC[C@@H](c)N, while
the other 18 molecules have presence due
to hash collisions.
In this sense, the various contributions should be assigned to each of the features represented in bit $955$
depending on their properties, but the limited hash length prevents it.
From this viewpoint, one would have to increase the hash length of ECFP
to resolve those differences.

One of the strengths of IGCS
is that we can use it to identify which features are
associated with a large loss between observed and
predicted responses.
See Figure~\ref{fig:10_61}(b).
The vertical distance of each point from diagonal line in this figure gives a good indicator of loss SHAP,
because the loss function is the usual MSE in this task.
We see that the greatest source of underestimation
for ZINC61 comes from the underestimation of
contributions from present features.
The origins of these differences can be from the sparsity of the counts of each feature over the molecules in test dataset.
This kind of information must be exploited to obtain accurate models.
However, there is no universal method to improve AI models.
We can try feature engineering, data augmentation, modify loss function or introduce attention mechanism et cetera.
This area is referred as eXplanatory Interactive machine Learning (XIL) in recent years. See \cite{teso2022leveraging} for a review of XIL.

\subsection{Number of Samples and ABC}
In Section \ref{sec:cern} we compared several methods
in terms of ABC using the CERN data.
The CS results in Section \ref{sec:cern}
could use exact CS because with only $d=16$
variables the cost was manageable.
In the DeepChem setting it is not possible to compute CS.
We are left to choose between IGCS and Monte Carlo (MC)
sampling.
That can be done by sampling from the $d!$ permutations
underlying one expression for Shapley value or
from other methods that just sample from the
$d2^{d-1}$ incremental values.
Sampling permutations is most
straightforward because it can use
uniform (unweighted) sampling.
Here we compare IGCS to Monte Carlo (MC) sampling
via ABC.  The IGCS values will have a bias
because $\psi_j\ne\phi_j$.  MC estimates $\hat\phi_j$
are unbiased but can have high variance.

A comparison is made in Table \ref{tb:abc_mcmc_dc} and
illustrated in Figure \ref{fig:dc_abc_mcmc}.
The averages are over 2000 held out molecules.
The error bars represent $\pm1$ standard error.
This is a standard error over the 2000 molecules
not over the MC simulations.
The results are compared to IGCS with two choices for the number of function values in the Riemann sum.
The IGCS reference lines are surrounded by bands shading
plus or minus one standard error (over 2000 molecules)
of the IGCS values.

The MC estimates show a clear upward trend; as more
MC samples are taken, the estimated ABCs trend steadily
upwards.  While the underlying MC $\hat\phi_j$ are unbiased,
the ABC estimates apply a nonlinear variable ranking
transformation to those estimates so that the MC-based
estimates $\widehat{\mathrm{ABC}}$ are biased.
From Figure~\ref{fig:dc_abc_mcmc} it is evident that
truly enormous MC sampling cost would be required for
the MC-based method to get ABC scores comparable to
those of IGCS.
MC sampling with about the same cost in time as IGCS
using $50$ gradient points to estimate the integral
yields an insertion ABC of $0.052$ compared to
$0.149$ with half a million Monte Carlo samples.
When ABC methods are used to demonstrate a good ordering,
we can see that MC can have severe difficulty.

\begin{table}[t]
    \centering
    \begin{tabular}{rccc}
        \toprule
        Number of & Insertion ABC & Deletion ABC  & Time        \\
        Samples   & (Std. error)  & (Std. error)  & (sec./data) \\
        \midrule
        2049      & 0.052 (0.016) & 0.063 (0.016) & $\phz$0.517 \\
        10000     & 0.028 (0.017) & 0.037 (0.017) & $\phz$1.387 \\
        100000    & 0.071 (0.016) & 0.060 (0.017) & 14.760      \\
        200000    & 0.106 (0.015) & 0.114 (0.016) & 29.154      \\
        300000    & 0.128 (0.014) & 0.144 (0.015) & 44.720      \\
        400000    & 0.146 (0.014) & 0.169 (0.015) & 59.618      \\
        500000    & 0.149 (0.014) & 0.179 (0.015) & 74.503      \\
        \midrule
        50 steps
        IGCS      & 0.318 (0.010) & 0.637 (0.011) & $\phz$0.516 \\
        200 steps
        IGCS      & 0.330 (0.010) & 0.664 (0.011) & $\phz$0.661 \\
        \bottomrule
    \end{tabular}
    \caption{The ABCs with different sampling number in CS and IGCS of ZINC15 data.
        The numbers in the parenthesis represent the standard error in 2000 data.}
    \label{tb:abc_mcmc_dc}
\end{table}

\begin{figure}[t]
    \includegraphics[width=\linewidth]{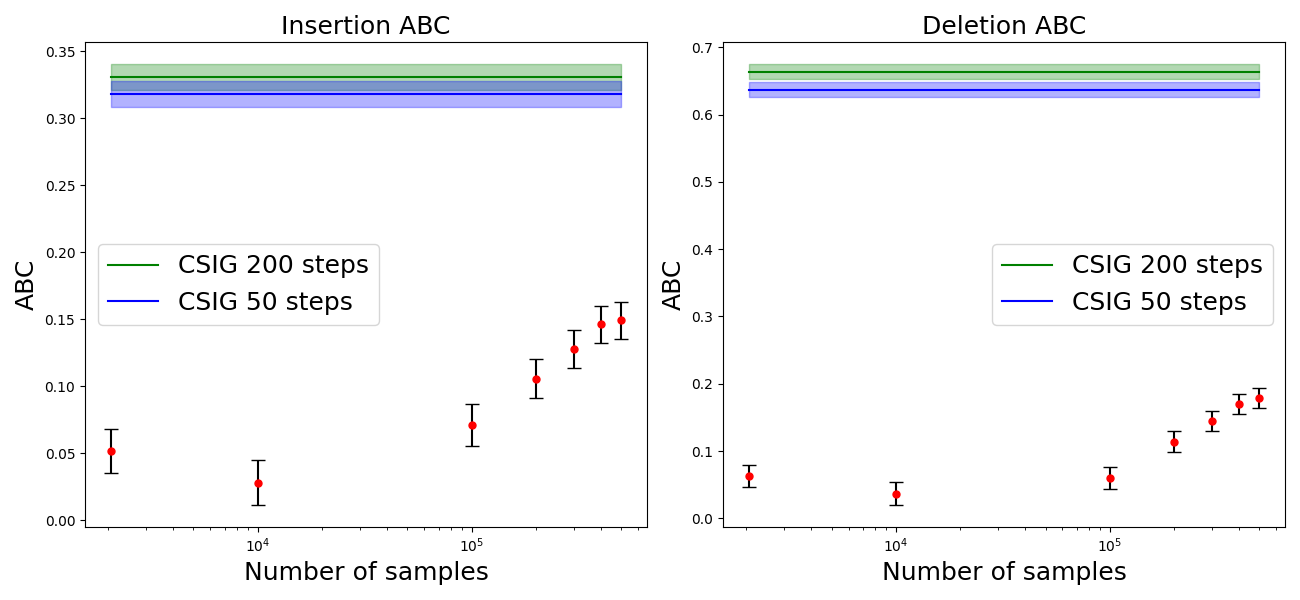}
    \caption{The ABCs with different sampling number in CS of ZINC15.}
    \label{fig:dc_abc_mcmc}
\end{figure}

The computation time is measured in a server with Intel(R) Xeon(R) Gold 6132 CPU @ 2.60GHz and Tesla V100-SXM2.

\section{Conclusion}
\label{sec:conclusion}

We have introduced a new variable importance
measure, IGCS, based on Aumann-Shapley theory.
Like CS it is model-free, but unlike 
CS it
can easily scale to high dimensional settings
without incurring exponential cost.
IGCS has a value function on $[0,1]^d$ that
is essentially multilinear over all but a negligible portion
of the unit cube when $d\gg\log(n)$.
That then makes it close to CS in the high
dimensional settings where exponential cost
rules out exact CS computations.
The disadvantage of IGCS is the same as CS: the
user needs to define what it means for two
levels of a variable to be similar.  While this
is a burden for non-binary variables, a user-defined
notion of similarity is more transparent than having
a notion of similarity that emerges from some other
black box.

This model-free method is useful in settings
where the prediction function $f$ is not available
either because it is proprietary or because it is
prohibitively expensive.  It can also be used
on raw data or on residuals from a model.

\section*{Acknowledgement}
This work was supported by the U.S.\ National Science
Foundation grants IIS-1837931 and DMS-2152780
and by Hitachi, Ltd.  We thank Benjamin Seiler and Masashi Egi for
helpful comments.

\bibliographystyle{apalike}
\bibliography{csig}

\appendix
\section{Detailed Model Descriptions}\label{sec:modeldetails}
This appendix provides some background details on the experiments conducted in this article.

\subsection{CERN Electron Collision Data}\label{sec:cerncolldata}
The hyperparameters for the model used in Section~\ref{sec:cernexp} are
given in Table \ref{tb:modelcern}.
They were obtained from a
hyperparameter search using Optuna from \cite{akiba2019optuna}.
Each  intermediate layer is  a parametric ReLU with dropout. The dropout ratio is common to all of the layers.
The model is trained with Huber loss.
The model is overall very accurate but the very highest values are systematically underestimated,
as seen in \cite{hama2022deletion}
which also used this model.

\begin{table}[htbp]
    \centering
    \begin{tabular}{ll}
        \toprule
        Hyperparameter    & Value                         \\
        \midrule
        Dropout Ratio     & 0.11604                       \\
        Learning Rate     & $1.9163 \times 10^{-4}$       \\
        Number of Neurons & [509--421--65--368--122--477] \\
        Huber Parameter   & $1.0$                         \\
        \bottomrule
    \end{tabular}
    \caption{Parameters of the MLP model for the CERN Electron Collision Data.}
    \label{tb:modelcern}
\end{table}

\subsection{ZINC15 Data}\label{sec:logp}
The model used in Section~\ref{sec:logp_exp} is trained with RobustMultitaskRegressor of \cite{ramsundar2017multitask}
by three tasks annotated in ZINC15 in DeepChem.
The model is trained with mean squared loss.
The hyperparameters for the model are given in Table \ref{tb:modellogp}.
The number of nodes in a hidden layer is determined from a  search using Optuna from \cite{akiba2019optuna}.
The learning rate is exponentially decaying from $0.01$
with a decay rate of $0.8$ per $1000$ steps.

\begin{table}[htbp]
    \centering
    \begin{tabular}{ll}
        \toprule
        Hyperparameter                      & Value  \\
        \midrule
        Dropout Ratio                       & 0.25   \\
        Number of Neurons in a hidden layer & [3935] \\
        \bottomrule
    \end{tabular}
    \caption{Parameters of the MLP model that estimate logP values in ZINC15 from ECFP vectors.}
    \label{tb:modellogp}
\end{table}

\section{Empirical Gaussian Kernel Weight}\label{app:egkw}
Here, we describe our
modification of the Gaussian kernel weight (GKW) method
that we used in Section \ref{sec:experiment}.
The Gaussian kernel weight method of \cite{aas2021explaining}
is not model free.  Here we describe a model-free
adaptation of it.

The value function in GKW takes the form
\begin{align}\label{eq:originalgkw}
    \nu(u) = \frac{\sum_{i=1}^ns_u(\bsx_i,\bsx_t)f(\bsx_{i,u}{:}\bsx_{t,-u})}{\sum_{i=1}^ns_u(\bsx_i,\bsx_t)}
\end{align}
for a weight function
$$
    s_u (\bsx_i, \bsx_t) = \exp\left(- \frac{D_u^2 (\bsx_i, \bsx_t)}{2\sigma^2}\right)
$$
where
$$
    D_u (\bsx_i, \bsx_t) =\sqrt{ \frac{(\bsx_{i,u}- \bsx_{t,u})^\tran \Sigma_{uu}^{-1} (\bsx_{i,u}- \bsx_{t,u})}{|u|}}.
$$
The function $D_u$ is a scaled Mahalanobis distance
from $\bsx_{iu}$ to $\bsx_{tu}$.  It uses
$\Sigma_{uu}$ which is the $|u|\times |u|$ covariance
matrix among the $\bsx_{iu}$ values.
The distances and weights are measured in data space, not in indicator space, so they are available for real-valued features but
perhaps not for other features.

GKW has a parameter $\sigma>0$.
The authors of \cite{aas2021explaining} recommend choosing
it via AIC.  That is expensive and they also report that
a default such as $\sigma=0.1$ works well in practice.

Before modifying GKW we make two observations.
First, the normalization by $|u|$ within $D_u$ is unusual
and it means that the distance from $\bsx_{iu}$ to $\bsx_{tu}$
could actually decrease when we incorporate another
variable, replacing $u$ by $u\cup\{j\}$.
We are not able to say whether this is a strength
or weakness of GKW, just that it is an interesting
propoerty. Second, the value $\sigma=0.1$
gave us some concern. It makes $s_u(\bsx_i,\bsx_t)$
equal to $\exp(-50S^2_u)$ where $S^2_u$ is a standardized
mean square difference.  It seems that this would then
put the vast majority of the weight on the single closest
point to $\bsx_t$, namely $\bsx_t$ itself.  To address
this concern we tried a larger values of $\sigma$
($1.0$ and $10.0$) and found
them to perform worse than $\sigma=0.1$ in our adaptation
and so we work with the recommended default.

The GKW formula~\eqref{eq:originalgkw} is of interventional
type.  It uses hybrid points that are not observed.
This means it cannot be used in a model-free setting.
It may also reference impossible input combinations.
We replace the hybrid points in~\eqref{eq:originalgkw}
by observed data points, getting
\begin{align}\label{eq:ourgkw}
    \nu(u) = \frac{\sum_{i=1}^ns_u(\bsx_i,\bsx_t)f(\bsx_i)}{\sum_{i=1}^ns_u(\bsx_i,\bsx_t)}.
\end{align}
We are then able to derive Shapley values
from this model-free measure.

\section{Higher Order Taylor Expansion}\label{sec:higherorder}

We saw in Section~\ref{sec:asymptote}
that for the overwhelming majority of the unit
cube $[0,1]^d$ and of its corners $\{0,1\}^d$, the
Taylor expansion~\eqref{eq:formaltaylor} of the CSIG
function is convergent.  The first order term is
one where integrated gradients give the same values
as Shapley. It is then instructive to carry out the Taylor expansion
to one more term to see how a mismatch
between CS and IGCS arises.  Doing that we get
\begin{align*}
    \nu(\bsz)
     & \approx f(\bsx_t)
    +\sum_{i\ne t}f(\bsx_i)\prod_{j\in J_i}(1-z_j)
    -f(\bsx_t)\sum_{i\ne t}\prod_{j\in J_i}(1-z_j) \\
     & \quad- \sum_{i\ne t}
    f(\bsx_i)\sum_{i'\ne t}\prod_{j\in J_i}(1-z_j)
    \prod_{j\in J_{i'}}(1-z_j).
\end{align*}
There are two new terms. The first new term also
involves products over $j\in J_i$ of $1-z_j$
and so it also is one for which IGCS matches
CS.

The second new term need not be of that form.
It does however involve two such products
of factors $1-z_j$ and we know that each such
product is typically smaller than $\epsilon$.
That term is then $O(n^2\epsilon^2)$
over most of $[0,1]^d$
while the other terms are $O(n\epsilon)$.
For large $d$ it is easy to make $n\epsilon$ negligible.
For instance with $\epsilon = 1/n^2$ we get
$\Pr(\bsz\in H_\epsilon)\le n^4\exp( -\lfloor ad\rfloor/4)$.
This is negligible for $d\gg (16/a)\log(n)$.

The second new term is a weighted sum of
$$
    g_{i,i'}(\bsz)=\prod_{j\in J_i}(1-z_j)\prod_{j\in J_{i'}}(1-z_j)
    =\prod_{j\in J_i\cap J_{i'}}(1-z_j)
    \times \prod_{j\in J_i\cup J_{i'}}(1-z_j).
$$

It is the factors $j\in J_i\cap J_{i'}$ that
make IGCS differ from CS.
The cohort Shapley value for $g_{i,i'}$ is
$$
    \frac1{|J_i\cup J_{i'}|}\times 1\{j\in J_i\cup J_{i'}\}.
$$
This follows from symmetry and the fact that
$g_{i,i'}$ only takes values $0$ or $1$ at
the corners of the cube.

To find the IGCS value for $g_{i,i'}$, note that
\begin{align*}
    g_{i,i'}(\bsz)
     & = \prod_{j\in J_i\cap J_{i'}}(1-z_j)^2
    \prod_{j\in J_i\triangle J_{i'}}(1-z_j)
\end{align*}
where $\triangle$ denotes the symmetric difference.
For $j'\in J_i\cap J_{i'}$
\begin{align*}
    \frac{\partial}{\partial z_{j'}}g_{i,i'}(\bsz)
     & = -2\prod_{j\in J_i\cap J_{i'}\setminus\{j'\}}(1-z_j)^2
    \prod_{j\in J_i\triangle J_{i'}\cup\{j'\}}(1-z_j),\quad\text{so} \\
    \frac{\partial}{\partial z_{j'}}g_{i,i'}(z\bsone)
     & =-2(1-z)^{2|J_i\cap J_{i'}|+|J_i\triangle J_{i'}|-1}
\end{align*}
which has IGCS value
$$
    \frac2{2|J_i\cap J_{i'}|+|J_i\triangle J_{i'}|}.
$$
Similarly for $j'\in J_i\triangle J_{i'}$
\begin{align*}
    \frac{\partial}{\partial z_{j'}}g_{i,i'}(\bsz)
     & = -1\prod_{j\in J_i\cap J_{i'}}(1-z_j)^2
    \prod_{j\in J_i\triangle J_{i'}\setminus\{j'\}}(1-z_j),\quad\text{so} \\
    \frac{\partial}{\partial z_{j'}}g_{i,i'}(z\bsone)
     & =-(1-z)^{2|J_i\cap J_{i'}|+|J_i\triangle J_{i'}|-1}
\end{align*}
which has IGCS value
$$
    \frac1{2|J_i\cap J_{i'}|+|J_i\triangle J_{i'}|}.
$$
The consequence is that IGCS gives extra weight
to $j\in J_i\cup J_{i'}$ and less
to $j\in J_i\triangle J_{i'}$.
We expect that any given variable $j$ will
be in the double weight grouping for some
pairs $i,i'$ and in the single weight grouping
for others.  A variable $j$ where $x_{ij}$ is
rarely dissimilar to $x_{tj}$ is less likely
to get those double weightings.

The above argument has not used the full generality
of Corollary~\ref{cor:generalmatch}. Approximation
within the wider class of functions considered there
can only reduce the differences $|\psi_j-\phi_j|$
between IGCS and CS compared to the multilinear
functions in our Taylor expansion.

\section{Deletion and Insertion Measures}\label{sec:abc}

As noted earlier, there is no ground truth for variable importance.
Instead there are multiple definitions of what
makes a variable important and choosing a definition
involves some tradeoffs.
It is however possible to compute the value of a proxy measure
derived from an area under the curve (AUC) quantity
defined by \cite{petsiuk2018rise}.  Our presentation
of that measure is
based on the analysis of AUC in \cite{hama2022deletion}.

The proxy measure is about the quality with which the
variables' importances can be ranked. For an interventional
method like baseline Shapley, the quality of a ranking can
be measured as follows. We sort the variables from those
with the largest $\phi_j$ to those with the smallest $\phi_j$.
For $k=0,1,\dots,d$ let $\tilde\bsx_k$ be a hybrid point
with $\tilde x_{kj}=x_{tj}$ if variable $j$ is among the $k$
variables with the largest $\phi_j$ and let $\tilde x_{kj}=x_{bj}$
otherwise.  Then let $\tilde y_k = f(\tilde \bsx_k)$
and consider the piecewise linear curve
going through the points $(k,\tilde y_k)$ over the interval $[0,d]$.
The AUC criterion is the (signed) area under this curve.
If the variables have been well ranked then the AUC
of the insertion curve will be large.
\cite{hama2022deletion} describe an area between the curves
(ABC) that subtracts the area under a straight line
connecting $(0,\tilde y_0)$ to $(d,\tilde y_k)$.
If one method has a better AUC than another it will
also have a better ABC. That paper argues that ABC
is better suited to regression problems than AUC
which was devised for classification.

There is a similar deletion curve
found by changing the variables in the reverse order.
\cite{hama2022deletion} define an ABC for deletion as
the signed area under the straight line minus the area
under the deletion curve.

Ranking is not the same as estimating the $\phi_j$.
\cite{hama2022deletion} show an example where
ranking by the true Shapley values can give a lesser
ABC than some other ranking.  They also show that for
certain simple models, such as logistic regression and
naive Bayes, that ranking by Shapley value does indeed maximize
the ABC.  Since those models often perform well, we expect
that more sophisticated models will often make similar
predictions and then ABC will be a useful scorecard.

On our data examples it is impossible to compute the
true Shapley values so we use ABC measures as a guide
to compare CS to IGCS and some other methods.

Just as there is more than one way to define variable
importance, there is more than one way to define an ABC
measure to rank variable importance measures.
Instead of the interventional approach above one can
replace $\tilde y_k$ by the average value of $f(\bsx_i)$
over the cohort of points that are similar to $\bsx_t$
for the $k$ variables with the greatest value of $\phi_j$.
This provides a conditional ABC measure as an alternative
to an interventional ABC measure.

We expect on intuitive grounds that a method which attempts
to compute interventional Shapley values will attain a better
interventional ABC value than one that attempts to compute
conditional Shapley values and vice versa. While we do not
prove that this must happen,
it did happen in our numerical examples.
We have a strong preference to avoid interventional measures
as they can require evaluating $f$ at some wildly unrealistic
input values, but other researchers might accept those measures.

\end{document}